\documentclass[11pt]{article}

\usepackage[final]{acl}

\usepackage{times}
\usepackage{latexsym}

\usepackage[T1]{fontenc}

\usepackage[utf8]{inputenc}

\usepackage{microtype}

\usepackage{inconsolata}

\usepackage{graphicx}
\usepackage[ruled,linesnumbered,vlined]{algorithm2e}
\usepackage{algpseudocode}
\usepackage{amsmath}  
\usepackage{amssymb}  
\usepackage{bm}
\usepackage{subfigure}
\usepackage{caption}
\usepackage{multirow}
\newtheorem{definition}{Definition}[section]
\newtheorem{proof}{Proof}[section]

\newtheorem{example}{Example}[section]
\newtheorem{lemma}{Lemma}[section]

\usepackage{array}
\usepackage{enumitem}
\usepackage{hyperref}
\usepackage{cleveref}
\usepackage{url}
\usepackage{stfloats}
\usepackage{xcolor}
\usepackage{colortbl}
\usepackage{tabularx}
\usepackage{booktabs} 

\definecolor{lightblue}{rgb}{0.68, 0.85, 0.90}

\newcommand{\model}{StruProKGR}

\title{\model{}: A Structural and Probabilistic Framework for Sparse Knowledge Graph Reasoning}

\author{
    Yucan Guo,
    Saiping Guan\thanks{Corresponding authors.},
    Miao Su,
    Jiyao Wei,\\
    \textbf{Xiaolong Jin$^{\dagger}$,
    Jiafeng Guo,
    Xueqi Cheng}
     \\
    \textsuperscript{1}State Key Laboratory of AI Safety \\
    \textsuperscript{2}Institute of Computing Technology, Chinese Academy of Sciences \\
    \textsuperscript{3}University of Chinese Academy of Sciences \\
    \small
    \texttt{\{guoyucan23z, guansaiping, sumiao22z, weijiyao20z, jinxiaolong, guojiafeng, cxq\}@ict.ac.cn}\\
}

\begin{document}
\begin{sloppypar}
\maketitle
\begin{abstract}
Sparse Knowledge Graphs (KGs) are commonly encountered in real-world applications, where knowledge is often incomplete or limited. Sparse KG reasoning, the task of inferring missing knowledge over sparse KGs, is inherently challenging due to the scarcity of knowledge and the difficulty of capturing relational patterns in sparse scenarios. Among all sparse KG reasoning methods, path-based ones have attracted plenty of attention due to their interpretability. 
Existing path-based methods typically rely on computationally intensive random walks to collect paths, producing paths of variable quality.
Additionally, these methods fail to leverage the structured nature of graphs by treating paths independently. To address these shortcomings, we propose a \underline{Stru}ctural and \underline{Pro}babilistic framework named \model{}, tailored for efficient and interpretable reasoning on sparse KGs. \model{} utilizes a distance-guided path collection mechanism to significantly reduce computational costs while exploring more relevant paths. It further enhances the reasoning process by incorporating structural information through probabilistic path aggregation, which prioritizes paths that reinforce each other. Extensive experiments on five sparse KG reasoning benchmarks reveal that \model{} surpasses existing path-based methods in both effectiveness and efficiency, providing an effective, efficient, and interpretable solution for sparse KG reasoning.\footnote{The code is publicly available at \url{https://github.com/YucanGuo/StruProKGR}.} 
\end{abstract}

\section{Introduction}
\label{sec:Introduction}
\begin{figure}[htbp]
    \centering
    \includegraphics[width=\linewidth]{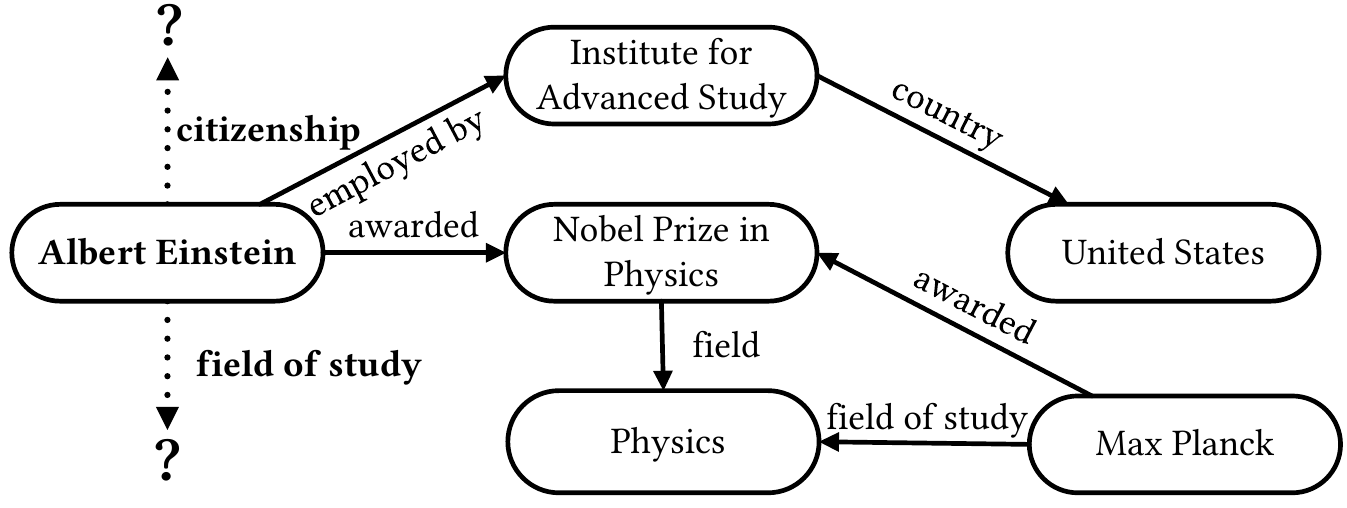}
    \caption{An example illustrating sparse KG reasoning.\label{fig:KGR_example}}
\end{figure}
Knowledge Graphs (KGs) contain facts in the form of triples $(head\ entity, relation, tail\ entity)$, denoted as $(h,r,t)$. They support a variety of downstream applications, including question answering~\cite{Agarwal2024symkgqa, Liu2025ontology}, recommender systems~\cite{Wang2024unleashing, Wang2025knowledge}, and information retrieval~\cite{Gutierrez2024hipporag, Cai2025simgrag}. 
In real-world situations, KGs often exhibit sparsity, as many triples are missing due to incomplete knowledge collection.  For example, considering Freebase~\cite{Bollacker2008freebase}, the well-known open-source KG, 71\% of individuals in Freebase have no recorded place of birth, and 75\% have no identified nationality~\cite{Dong2014knowledge}. 
Sparse KG reasoning, the task of inferring missing knowledge over sparse KGs, is crucial for uncovering valuable insights in knowledge-scarce scenarios. 

\Cref{fig:KGR_example} shows an example of sparse KG reasoning over a highly incomplete KG with six entities and five relations. Solid arrows denote facts that are observed, while the dotted arrows highlight the relations that should be completed, i.e., $(Albert\ Einstein,\ citizenship,\ ?)$ and $(Albert\ Einstein,\ field\ of\ study,\ ?)$. These two relations do not appear explicitly in the observed graph but can be inferred from other information within it. The $citizenship$ of $Albert\ Einstein$ is likely to be deduced from the institution he employed by, while the $field\ of\ study$ can often be inferred by propagating the $field$ from the prize to its laureate. 
However, the scarcity of knowledge and the complexity of understanding intricate relational patterns render this task exceptionally challenging. Existing KG reasoning methods often fall short in sparsity scenarios~\cite{Pujara2017sparsity}, highlighting the need for approaches specifically designed for sparse KGs.

Sparse KG reasoning methods fall into three major categories: embedding-based, rule-based, and path-based methods.
Embedding-based methods~\cite{Zhang2022rethinking, Tan2023kracl, Chen2024hogrn} encode entities and relations into continuous spaces and yield strong predictive performance, but they are often opaque.
Rule-based methods~\cite{Meilicke2020reinforced, Sun2023effective} mine symbolic rules from KGs and provide interpretability, yet they suffer from scalability issues and costly rule mining.
Path-based methods~\cite{Lao2011random, Lv2020dynamic, Guan2024look}, by contrast, trace explicit relational paths without requiring meticulously designed rules, offering transparency that is crucial for trustworthy knowledge inference.
Existing path-based methods typically rely on either random walk-based~\cite{Lao2011random, Gardner2015efficient, Guan2024look} or reinforcement learning (RL)-based strategies~\cite{Das2018go, Lv2020dynamic} for path collection, followed by path reasoning over the collected paths.
However, both path collection and path reasoning stages remain challenging. Random walk-based approaches are computationally expensive on large KGs and often generate low-relevance paths due to stochastic exploration. RL-based approaches guide path selection via learned policies but compromise interpretability.
Moreover, most path-based methods reason over paths independently, ignoring the structural dependencies among paths in sparse KGs. This assumption prevents models from capturing collective relational patterns and limits reasoning accuracy.

To overcome these challenges, we present \model{}, a novel path-based framework meticulously designed for effective, efficient, and interpretable reasoning over sparse KGs. 
\model{} introduces a distance-guided path collection mechanism that markedly reduces computational overhead compared to random walk-based methods. This approach leverages distances to the tail entity to prioritize paths that are most likely to contribute to accurate reasoning outcomes, thereby optimizing the exploration process of sparse KGs. By prioritizing paths with high relevance to target relations, this approach ensures both effectiveness and efficiency, addressing the scalability concerns of prior random walk-based techniques. 
Additionally, \model{} utilizes the structural properties of sparse KGs through a probabilistic path aggregation strategy during path reasoning. This approach considers the correlations among paths as a whole, resulting in more accurate inferences of missing knowledge while preserving the interpretability of path-based methods.

In summary, the contributions of this paper are as follows: 
\begin{itemize}[leftmargin=*]
    \item We present \model{}, a training-free path-based framework for effective, efficient, and interpretable reasoning over sparse KGs.
    \item We design a distance-guided path collection and a probabilistic path aggregation mechanism that reduce computational overhead while leveraging graph structure to enhance reasoning accuracy.
    \item Extensive experiments on five sparse KG benchmarks demonstrate the effectiveness and efficiency of \model{}.
\end{itemize}

\section{Related Work}
\label{sec:Related Work}
In this section, we review the related work of embedding-based, rule-based, and path-based sparse KG reasoning methods.

\smallskip
\noindent\textbf{Embedding-based Methods.}
Embedding-based methods learn vector representations for entities and relations in a KG, using these to score the plausibility of potential triples.
Early methods like TransE~\cite{Bordes2013translating} interpret relations as translations in vector space, while DistMult~\cite{Yang2015embedding} employs the bilinear objective to learn relational semantics. 
Several studies enhance embeddings by graph walks or relation paths to capture richer semantic information. 
RDF2Vec~\cite{Ristoski2016rdf2vec} learns entity embeddings from sequences collected by graph walks. PTransE~\cite{Lin2015modeling} incorporates multi-step relation paths into embedding learning to improve reasoning, while RPJE~\cite{Niu2020rule} integrates automatically mined Horn rules and paths into representation learning.
More advanced methods, such as ConvE~\cite{Dettmers2018convolutional} and TuckER~\cite{Balavzevic2019tucker}, leverage convolutional neural networks and tensor factorization to model complex interactions. 
Recent studies~\cite{Tan2023kracl, Chen2024hogrn} integrate graph context into their models to tackle the sparsity issue.
These methods often achieve strong prediction performance but suffer from limited interpretability and high computational costs due to representation learning.

\smallskip
\noindent\textbf{Rule-based Methods.}
Rule-based methods mine logical rules from KGs to infer new knowledge, offering clear explanations for predictions. 
AMIE~\cite{Galarraga2013amie} and AnyBURL~\cite{Meilicke2020reinforced} extract horn clauses to capture relational patterns, with AnyBURL incorporating RL to enhance rule mining. 
NTP~\cite{Rocktaschel2017end} integrates differentiable proving with subsymbolic representations, enabling logical rule induction through gradient-based optimization.
RLvLR~\cite{Omran2018scalable} presents a method that combines representation learning with closed path rule mining, using embeddings and sampling to handle large KGs. 
However, these methods face challenges in sparse KGs, where limited facts reduce rule coverage and reliability. Additionally, the rule mining process can become computationally intensive as the complexity of the rules increases. 

\smallskip
\noindent\textbf{Path-based Methods.}
Path-based KG reasoning methods collect and traverse relational paths to infer missing knowledge, typically using either random walk-based or RL-based strategies.
Random walk-based methods are pioneered by PRA~\cite{Lao2011random}, with later extensions such as Prob-CBR~\cite{Das2020probabilistic} introducing probabilistic case-based reasoning and LoGRe~\cite{Guan2024look} constructing a global relation-path schema to mitigate sparsity. However, as the KG complexity increases, the computational cost of random walks escalates exponentially, rendering these approaches impractical for large-scale KGs.
RL-based methods instead learn to navigate paths. DacKGR~\cite{Lv2020dynamic} expands the search space with dynamically added edges, SparKGR~\cite{Xia2022iterative} integrates rule-guided iterative refinement, and recent systems such as DT4KGR~\cite{Xia2024dt4kgr} and Hi-KnowE~\cite{Xie2024hierarchical} incorporate decision Transformers or hierarchical RL. Despite these advances, RL-based methods commonly require handcrafted rewards or external resources, which add complexity and undermine the interpretability central to path-based reasoning.

\section{Problem Statement}
\label{sec:Problem Definition}
We first introduce key concepts of sparse KGs, then formally define the sparse KG reasoning task.

\subsection{Preliminaries}
\begin{definition}[KG]
Given an entity set $\mathcal{E}$ and a relation set $\mathcal{R}$, a KG is a directed graph $\mathcal{G}=\{(h,r,t)|h, t\in\mathcal{E}, r\in\mathcal{R}\}$, where each entity $e\in\mathcal{E}$ belongs to an entity type $c \in \mathcal{C}$, each triple $(h,r,t)$ indicates that there is a relation $r$ from the head entity $h$ to the tail entity $t$. 
\end{definition}

A sparse KG~\cite{Lv2020dynamic} refers to a KG that is highly incomplete, meaning that many potential links and facts between entities are missing. This form of sparsity should be distinguished from the statistical rarity of specific relations.
In practice, sparsity manifests as low triple density and weak connectivity between entities, which significantly increases the difficulty of reasoning.
To address sparsity, path-based methods often leverage relational sequences, i.e., relation paths, which describe multi-hop connections between entities. Relation paths can be defined at different granularities~\cite{Guan2024look}: (1) \textbf{Type-specific relation paths} capture connections between entities of a given type and a target relation; (2) \textbf{Relation paths} generalize these by aggregating type-specific relation paths across multiple entity types.

\subsection{Problem Definition}
\begin{definition}[Sparse KG reasoning]
Given a sparse KG $\mathcal{G}_s$ and a query $(h,r,?)$, where $h \in \mathcal{E}$ is a head entity and $r \in \mathcal{R}$ is a relation, the task is to predict the missing tail entity $t \in \mathcal{E}$ such that $(h,r,t)$ is likely to hold in $\mathcal{G}_s$.
\end{definition}

Queries of the form $(?,r,t)$ can be equivalently handled by introducing the inverse relation $r^{-1}$ and reformulating the query as $(t,r^{-1},?)$. Hence, it suffices to study the $(h,r,?)$ case.

\begin{figure*}[t]
    \centering
    \includegraphics[width=\linewidth]{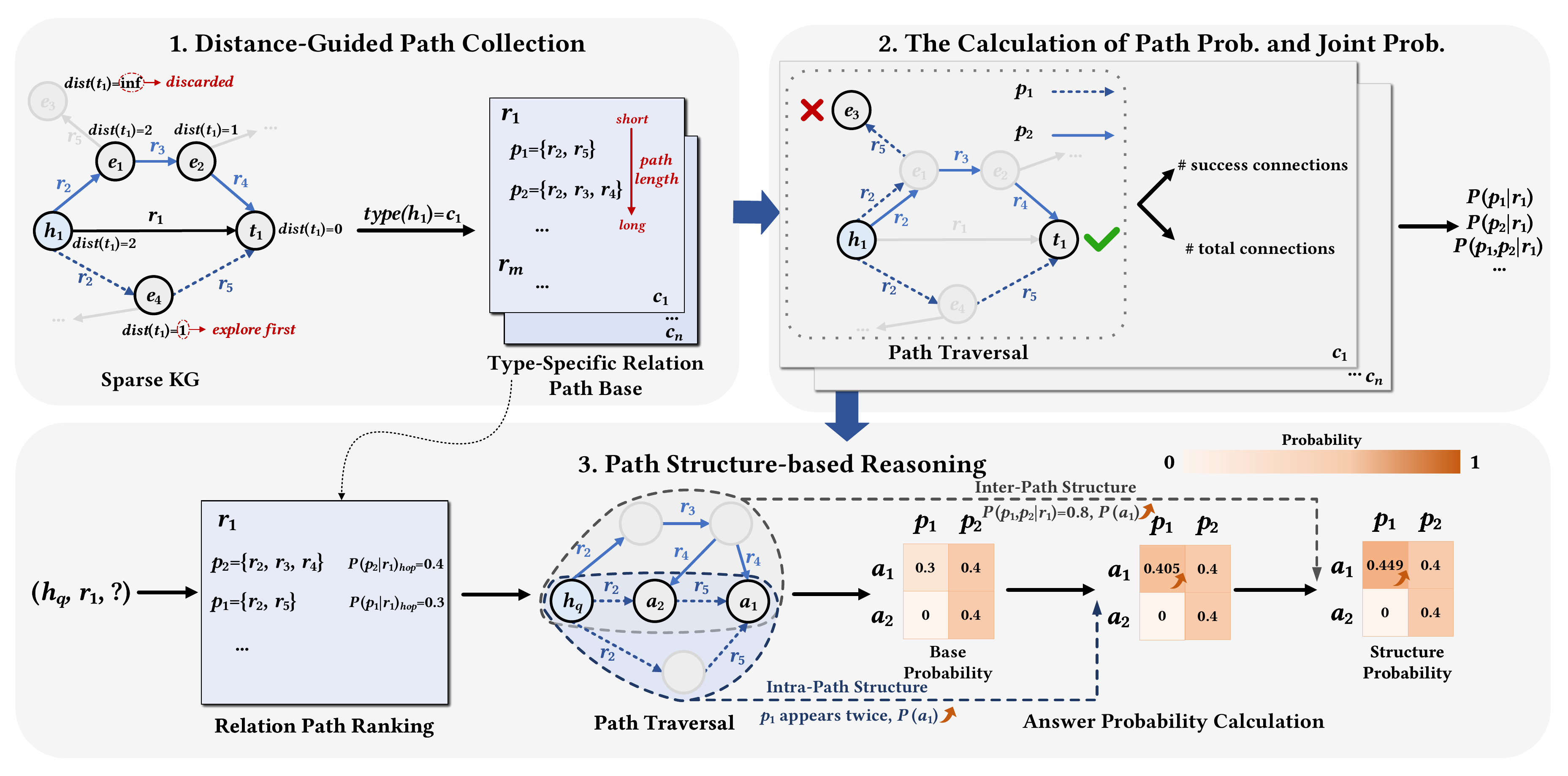}
    \caption{The proposed \model{} framework.\label{fig:framework}}
\end{figure*}
\section{Methodology}
\label{sec:Methodology}
In this section, we introduce the details of the proposed \model{} framework, and~\Cref{fig:framework} illustrates its architecture. 
Specifically, \model{} is designed with three main phases: distance-guided path collection (\S\ref{sec:Distance-Guided Path Collection}), the calculation of path probability and joint probability (\S\ref{sec:The Calculation of Path Prob. and Joint Prob.}), and path structure-based reasoning (\S\ref{sec:Path Structure-based Reasoning}).

\subsection{Distance-Guided Path Collection}
\label{sec:Distance-Guided Path Collection}
In path-based sparse KG reasoning methods, efficiently collecting relevant paths between entities is critical for accurate reasoning. We introduce a distance-guided path collection mechanism that enhances computational efficiency and path relevance.
Short paths, which are more likely to be rules in path-based methods, are prioritized by leveraging distance information to guide and prune the path collection process.
The distance-guided path collection phase adopts a Depth-First Search (DFS) procedure to collect type-specific relation paths, and leverages precomputed shortest-path information to prune the search space aggressively. 
Specifically, this phase consists of two steps, i.e., (1) distances precomputation, and (2) distance-guided path collection. 

\smallskip
\noindent\textbf{Distances Precomputation.} 
We perform a truncated Breadth‐First Search (BFS) from each entity $u\in\mathcal{E}$, up to depth $l_{max}$, to fill the matrix $dist[u][v]$, where $l_{max}$ is the maximum path length. 
This truncated BFS records the minimum number of hops from $u$ to every reachable $v$ with $dist[u][v]\leq l_{max}$. By doing so, once at initialization, we can quickly check whether a partial path still has a chance to reach the target, thereby avoiding unnecessary visits to paths that are guaranteed to fail.
Distances are stored in a sparse matrix structure containing only nodes within the bounded radius, so both time and space complexity scale with the local neighborhood size rather than the total number of entities.

\smallskip
\noindent\textbf{Distance-Guided Path Collection.} 
For each training triple $(h,r,t)$, we perform a DFS starting from the head entity $h$ to reach the tail entity $t$. The search incrementally builds a candidate type-specific relation path $path = [r_1, r_2, \dots, r_m]$ for entity type $c\in C$ of $h$, which records the sequence of relations traversed so far. At each expansion step, a current entity $u$ can only move to a neighbor $v$ if
\begin{equation}
dist[v][t] \leq l_{max} - len(path) - 1,
\end{equation}
where $len(path)$ denotes the length of the current path. 
This guarantees that $v$ can still reach $t$ within the remaining length budget, eliminating futile branches.
To further narrow down the search, at each step, only the top-$k$ neighbors ranked by $dist[v][t]$ are retained, focusing the search on the most promising candidates and avoiding the redundancy of exploring many low-quality paths.
The two-stage pruning above ensures that only those branches that can feasibly reach the target within the allotted steps are considered, and concentrates efforts on the nearest neighbors. 

Formally, the procedure outputs the set of all collected paths $\mathcal{P}$, where each type-specific subset $\mathcal{P}(c,r)$ contains paths for entity type $c$ and relation $r$.
A complete algorithmic description is provided in Appx.~\S\ref{appendix:DGPC}.

\subsection{The Calculation of Path Probability and Joint Probability}
\label{sec:The Calculation of Path Prob. and Joint Prob.}
To prepare for the path structure-based reasoning phase, we need to calculate the \emph{path probability} and \emph{joint path probability} of paths in a sparse KG. The path probability aggregates type-specific relation paths across entity types to form a unified measure for relation paths, while the joint probability evaluates path pairs.

\smallskip
\noindent\textbf{Path Probability.}
The probability of a relation path $p$ for a relation $r$, denoted $P(p|r)$, quantifies the precision of $p$ in connecting a head entity $h$ to a tail entity $t$. Since relation paths are aggregated from type-specific relation paths defined for specific entity types, we traverse the collected type-specific relation paths and aggregate them to form $P(p|r)$. 
It is defined as: 
\begin{equation}
P(p|r) = \frac{\sum_{c \in \mathcal{C}} S_c(r,p)}{\sum_{c \in \mathcal{C}} T_c(r,p)},
\end{equation}
where $S_c(r,p)$ is the number of occurrences for path $p \in \mathcal{P}(c,r)$ that successfully reach the correct tail entities, and $T_c(r,p)$ is the total number of entities reached.
This metric prioritizes paths that consistently yield accurate inferences, forming the foundation for reliable reasoning.

\smallskip
\noindent\textbf{Joint Path Probability.} The joint probability of two paths $p_i$ and $p_j$ for a relation $r$, denoted $P(p_i, p_j|r)$, measures the likelihood that both paths collaboratively infer the relation $r$ correctly and reflects the combined reliability of path pairs. It is defined as: 
\begin{equation}
P(p_i, p_j|r) = \frac{JS(r,p_i,p_j)}{JT(r,p_i,p_j)},
\end{equation}
where $JS(r,p_i,p_j)$ is the number of joint correct occurrences for path pair $(p_i, p_j)$, and $JT(r,p_i,p_j)$ is the total number of joint occurrences.

To calculate $P(p|r)$ and $P(p_i, p_j|r)$ efficiently, we propose a batch search-based path traversal algorithm, which is detailed in Appx.~\S\ref{appendix:Probability_Calculation}.

\subsection{Path Structure-based Reasoning}
\label{sec:Path Structure-based Reasoning}
Sparse KGs inherently contain complex structural properties that need to be taken into account. For example, a relation may be more likely to hold true when certain paths occur together, while it may be less likely to be true when specific combinations of paths coexist. 
Thus, in the final phase, \model{} models structural properties of paths in a probabilistic manner to conduct reasoning effectively. 
We categorize structural properties in sparse KGs into two main groups: \emph{intra-path structures}, which relate to the internal characteristics of individual path types, and \emph{inter-path structures}, which focus on the relationships between different types of paths.

\smallskip
\noindent\textbf{Relation Path Ranking.}  
Before reasoning, collected paths are ranked by incorporating a hop decay factor $\alpha^{\text{len}(p)-1}$ into their probability~\cite{Guan2024look}, where $\alpha \in (0, 1)$ and $\text{len}(p)$ is the path length. 
The adjusted probability is calculated by
\begin{equation}
    P(p|r)_{hop} = P(p|r) \cdot \alpha^{\text{len}(p) - 1},
\end{equation}
which provides a measurable assessment of path relevance by balancing informativeness and conciseness, and serves as the base probability for the subsequent path structure-based probability update.

\smallskip
\noindent\textbf{Intra-Path Structure Modeling.}  
Intra-path structure focuses on the repetitive occurrence of a single path type that reaches the same entity during reasoning, and we model the contribution of repetition in a diminishing way. 
Specifically, each additional occurrence of the same path contributes less to the overall probability, which is defined as follows:
\begin{equation}
    P(p|r)_{k} = \beta^{k-1} \cdot P(p|r)_{hop},
\end{equation}
where $k$ means the $k$-th occurrence of path $p$, and the diminishing factor $\beta \in (0, 1)$. 

Consequently, the probability of a path, taking into account the intra-path structure, can be expressed as follows: 
\begin{equation}
    P(p|r)_{intra} = 1-\prod_{i=1}^{T_p}\big( 1-P(p|r)_{i}\big),
\end{equation}
where $T_p$ represents the total occurrences of path $p$ from the same head entity to the same tail entity.

\smallskip
\noindent\textbf{Inter-Path Structure Modeling.}
Inter-path structure examines the relationships between different types of paths, which often exhibit complex interactions that influence the accuracy of inferring missing knowledge. 
To address this, we propose a probabilistic framework that models inter-path structures using path and joint probabilities.

\textit{Likelihood Ratio Calculation.} 
Bayes' theorem provides a principled way of updating prior beliefs in light of new evidence, where the strength of evidence is captured by a likelihood ratio~\cite{Joyce2021bayes}. Following this, we introduce a scalable approximation of the likelihood ratio that aggregates evidence from multiple paths:
\begin{equation}
\small
\begin{aligned} 
  &LR(p_i, \mathcal{P}(r) \setminus \{p_i\}) = \\
  &\frac{\sum_{p_j} P(p_i, p_{j} | r)}{\sum_{p_j} [P(p_i | r) + P(p_j | r) - P(p_i | r) \cdot P(p_j | r)]},
\end{aligned}
\end{equation}
where $p_j \in \mathcal{P}(r) \setminus \{p_i\}$ and subject to the condition $P(p_j|r)_{hop} > P(p_i|r)_{hop}$.
The ratio compares the observed joint correctness to the expected correctness under independence. A value greater than 1 suggests that the paths are more likely to be correct together than independently, indicating collaboration, while a value less than 1 suggests inhibition. A mathematical proof for the approximation is provided in Appx.~\S\ref{appendix:proof_for_inter_path_structure}.

\textit{Updating Inter-Path Probabilities.} 
After calculating the likelihood ratio, path probabilities considering the inter-path structure can be determined using the odds form of Bayes' theorem.
The prior odds for path $p_i$ are calculated as:
\begin{equation}
    O(p_i) = \frac{P(p_i | r)_{intra}}{1 - P(p_i | r)_{intra}},
\end{equation}
reflecting the initial confidence in $p_i$ before accounting for interactions. The posterior odds, adjusted based on evidence from other paths, are then:
\begin{multline}
    O(p_i | \mathcal{P}(r) \setminus \{p_i\}) = \\
    O(p_i) \cdot LR(p_i, \mathcal{P}(r) \setminus \{p_i\}).
\end{multline}
This step updates our confidence in $p_i$ by incorporating the influence of other paths in $\mathcal{P}(r)$.
To obtain the updated probability, we convert the posterior odds back to a probability:
\begin{equation}
    P(p_i | r)_{inter} = \frac{O(p_i | \mathcal{P}(r) \setminus \{p_i\})}{1 + O(p_i | \mathcal{P}(r) \setminus \{p_i\})}.
\end{equation}

\noindent\textbf{Reasoning Workflow.} 
Bringing together intra- and inter-path structures, \model{} executes reasoning in three steps.  
First, for a given query $(h,r,?)$, it selects the top-$N_{top}$ relation paths from $\mathcal{P}(r)$ and traverses them starting from the head entity $h$, thereby gathering a set of candidate answers.  
Second, for each candidate, path probabilities are adjusted by incorporating both intra-path repetition effects and inter-path interactions, yielding $P(p|r)_{inter}$ for all contributing paths.  
Third, the probability of each candidate $a$ is aggregated from the set of paths $\mathcal{P}_a$ as
\begin{equation}
    P(a) = 1 - \prod_{p \in \mathcal{P}_a} \big(1 - P(p|r)_{inter}\big).
\end{equation}
By doing so, evidence from multiple supporting paths is combined, thereby jointly increasing the confidence in candidate $a$. 
The candidates are then ranked by their final scores, producing the reasoning output.  
The algorithmic details and pseudocode are provided in Appx.~\S\ref{appendix:reasoning_algorithm}.

\section{Experiments}
\label{sec:Experiments}
In this section, we conduct extensive experiments to verify the effectiveness, efficiency, and interpretability of \model{} across five widely recognized benchmark datasets for sparse KGs. The empirical findings are aimed at addressing the following key research questions: 
\textbf{RQ1.} How does \model{} perform against existing state-of-the-art methods in sparse KG reasoning?
\textbf{RQ2.} To what extent does the distance-guided path collection phase enhance both the effectiveness of the overall reasoning process and the efficiency of the path collection process?
\textbf{RQ3.} What are the impacts of different components in the path structure-based reasoning phase, including intra-path structure and inter-path structure?
\textbf{RQ4.} Despite its theoretical interpretability, does \model{} demonstrate interpretability in the practical reasoning process?

\subsection{Experimental Setup}
\subsubsection{Datasets}
We utilize five benchmark datasets that are widely used for sparse KG reasoning tasks: FB15K-237-10\%, FB15K-237-20\%, FB15K-237-50\%, NELL23K, and WD-singer~\cite{Lv2020dynamic}. 
These datasets vary in sparsity and domain characteristics, making them well-suited to test the ability of \model{} to leverage structural path information under knowledge-limited conditions.
\textbf{FB15K-237-10\%}, \textbf{FB15K-237-20\%}, and \textbf{FB15K-237-50\%} are subsampled versions of FB15K-237~\cite{Toutanova2015representing}, retaining 10\%, 20\%, and 50\% of the original triples, respectively.
\textbf{NELL23K} is a randomly sampled dataset from the NELL~\cite{Carlson2010toward} knowledge base.
\textbf{WD-singer} is a domain-specific Wikidata~\cite{Vrandevcic2014wikidata} subset focused on singer-related entities.
Detailed statistics on the benchmark datasets can be found in~\Cref{tab:dataset_stats}.

\begin{table}[t]
    \centering
    \setlength{\tabcolsep}{2pt}
    \small
    \caption{Statistics of the datasets used in experiments.}
    \label{tab:dataset_stats}
    \begin{tabular}{c|c|c|c|c|c}
        \toprule
        Dataset & \# Ent. & \# Rel. & \# Train & \# Valid & \# Test \\
        \midrule
        FB15K-237-10\% & 11,512 & 237 & 27,211 & 15,624 & 18,150\\
        FB15K-237-20\% & 13,166 & 237 & 54,423 & 16,963 & 19,776\\
        FB15K-237-50\% & 14,149 & 237 & 136,057 & 17,449 & 20,324\\
        NELL23K & 22,925 & 200 & 25,445 & 4,961 & 4,952\\
        WD-singer & 10,282 & 131 & 15,906 & 2,084 & 2,134\\
        \bottomrule
    \end{tabular}
\end{table}

\begin{table*}[htb]
	\setlength{\tabcolsep}{2.5pt}
	\centering
	\small
	\caption{Experimental results presented in terms of MRR and Hits@\{3, 10\} (\%). The best scores for rule-based methods (the second block) and path-based methods (the third block) are in bold, while the best scores for embedding-based methods (the first block) are underlined.\textsuperscript{2}}
	\makebox[\columnwidth][c]{
		\begin{tabular}{c|ccc|ccc|ccc|ccc|ccc}
			\toprule[1pt]
			\multirow{2}{*}{Method} &\multicolumn{3}{c|}{FB15K-237-10\%} &\multicolumn{3}{c|}{FB15K-237-20\%} &\multicolumn{3}{c|}{FB15K-237-50\%} &\multicolumn{3}{c|}{NELL23K} &\multicolumn{3}{c}{WD-singer}\\
            [-1em]\\
			\cline{2-16}\\[-0.8em]
            &MRR &H@3 &H@10 &MRR &H@3 &H@10 &MRR &H@3 &H@10 &MRR &H@3 &H@10 &MRR &H@3 &H@10\!\\
			\midrule[0.5pt]
            TransE &0.105 &15.9 &27.9 &0.123 &18.0 &31.3 &0.177 &23.4 &40.4 &0.084 &10.9 &24.7 &0.210 &32.1 &44.6\\
			TuckER &0.252 &27.2 &40.4 &\underline{0.268} &\underline{28.9} &\underline{42.8} &0.314 &\underline{34.2} &50.1 &0.276 &30.2 &46.7 &0.421 &47.1 &57.1\\
			ConvE &0.245 &26.2 &39.1 &0.261 &28.3 &41.8 &0.313 &\underline{34.2} &50.1 &0.276 &30.1 &46.4 &0.448 &47.8 &56.9\\
			NBFNet &0.241 &26.3 &38.8 &0.260 &27.8 &41.7 &\underline{0.316} &34.1 &\underline{50.3} &0.274 &28.9 &46.9 &0.453 &49.3 &\underline{58.9}\\
			KRACL &0.164 &17.0 &21.2 &0.170 &16.9 &19.8 &0.222 &26.8 &44.4 &0.158 &15.8 &27.6 &0.142 &13.4 &20.7\\  
			HoGRN &\underline{0.257} &\underline{27.5} &\underline{41.2} & - & - & - & - & - & - &\underline{0.292} &\underline{32.2} &\underline{49.1} & \underline{0.470} & \underline{51.0} & 57.8\\
			\midrule[0.75pt]
			NTP &0.083 &11.4 &16.9 &0.173 &16.1 &21.7 &0.222 &23.1 &30.7 &0.132 &14.9 &24.1 &0.292 &31.1 &44.2\\
			RLvLR &0.107 &12.2 &20.6 &0.132 &15.2 &27.1 &0.199 &20.8 &32.4 &0.152 &17.3 &25.0 &0.374 &32.0 &47.6\\
			AnyBURL &0.149 &15.5 &26.7 &0.164 &16.7 &29.3 &0.198 &21.3 &35.1 &0.176 &18.5 &25.2 &0.392 &34.1 &48.6\\
			\midrule[0.5pt]
			DacKGR &0.218 &23.9 &33.7 &0.242 &27.2 &38.9 &0.293 &32.0 &45.7 &0.197 &20.0 &31.6 &0.377 &42.1 &48.5\\
			SparKGR &0.228 &24.5 &35.0 &0.252 &27.7 &39.1 &0.292 &32.0 &46.2 &0.203 &22.2 &33.9 &0.393 &43.7 &50.7\\
			DT4KGR &- &- &- &0.254 &- &40.1 &0.297 &- &46.2 &- &- &- &- &-\\
			Hi-KnowE &0.224 &\textbf{25.5} &34.1 &0.247 &27.7 &38.1 &- &- &- &- &- &- &- &- &-\\
            LoGRe &0.228 &24.5 &36.2 &0.261 &28.0 &41.3 &0.297 &32.7 &46.4 &0.259 &27.9 &41.7 &0.459 &48.9 &54.5\\
			\midrule[0.5pt]
            \rowcolor{cyan!10}
			\model{} &\textbf{0.234} &25.2 &\textbf{37.3} &\textbf{0.267} &\textbf{28.8} &\textbf{42.1} &\textbf{0.304} &\textbf{33.3} &\textbf{47.6} &\textbf{0.262} &\textbf{28.5} &\textbf{42.7} &\textbf{0.461} &\textbf{49.8} &\textbf{55.6}\\
			\bottomrule[1pt]
		\end{tabular}
	}
	\label{table:overall performance}
\end{table*}
\begin{table*}[htbp]
	\setlength{\tabcolsep}{2.5pt}
	\centering
	\small
	\caption{Effectiveness analysis of path collection algorithms presented in terms of MRR and Hits@\{3, 10\} (\%).}
	\makebox[\columnwidth][c]{
		\begin{tabular}{c|ccc|ccc|ccc|ccc|ccc}
			\toprule[1pt]
			\multirow{2}{*}{Method} &\multicolumn{3}{c|}{FB15K-237-10\%} &\multicolumn{3}{c|}{FB15K-237-20\%} &\multicolumn{3}{c|}{FB15K-237-50\%} &\multicolumn{3}{c|}{NELL23K} &\multicolumn{3}{c}{WD-singer}\\
            [-1em]\\
			\cline{2-16}\\[-0.8em]
            &MRR &H@3 &H@10 &MRR &H@3 &H@10 &MRR &H@3 &H@10 &MRR &H@3 &H@10 &MRR &H@3 &H@10\!\\
			\midrule[0.5pt]
            \model{}$_{\text{RW}}$ &0.226 &24.3 &36.2 &0.261 &28.2 &41.5 &0.298 &32.9 &46.7 &0.260 &28.1 &\textbf{42.8} &0.459 &49.4 &55.0\\
			\model{} &\textbf{0.234} &\textbf{25.2} &\textbf{37.3} &\textbf{0.267} &\textbf{28.8} &\textbf{42.1} &\textbf{0.304} &\textbf{33.3} &\textbf{47.6} &\textbf{0.262} &\textbf{28.5} &42.7 &\textbf{0.461} &\textbf{49.8} &\textbf{55.6}\\
			\bottomrule[1pt]
		\end{tabular}
	}
	\label{table:path collection}
\end{table*}
\subsubsection{Implementation Details}
We conduct a grid search to determine the optimal hyperparameter for the maximum branch number, $k \in \{3,5,10,15,20,30\}$. The applied settings are: $k=15$ for FB15K-237-10\%, $k=5$ for FB15K-237-20\%, $k=3$ for FB15K-237-50\%, $k=30$ for NELL23K and WD-singer.
The diminishing factor $\beta$ is set to 0.5, and the inter-path structure is only taken into account for the top 200 paths to save time. For the other hyperparameters, we adhere to the settings in LoGRe~\cite{Guan2024look}.

\subsubsection{Baselines and Evaluation Metrics}
\textbf{Baselines.} 
We compare \model{} against a diverse set of state-of-the-art methods. For path-based methods, we compared with DacKGR~\cite{Lv2020dynamic}, SparKGR~\cite{Xia2022iterative}, DT4KGR~\cite{Xia2024dt4kgr}, Hi-KnowE~\cite{Xie2024hierarchical}, and LoGRe~\cite{Guan2024look}. For rule-based methods, we compared with NTP~\cite{Rocktaschel2017end}, RLvLR~\cite{Omran2018scalable}, and AnyBURL~\cite{Meilicke2020reinforced}. For embedding-based methods, we compared with TransE~\cite{Bordes2013translating}, TuckER~\cite{Balavzevic2019tucker}, ConvE~\cite{Dettmers2018convolutional}, NBFNet~\cite{Zhu2021neural}, KRACL~\cite{Tan2023kracl}, and HoGRN~\cite{Chen2024hogrn}.

\smallskip
\noindent \textbf{Evaluation Metrics.} 
We adopt standard metrics for sparse KG reasoning, i.e., Mean Reciprocal Rank (MRR) and Hits@K (K=3,10). Higher scores indicate a better ranking of correct answers.

\setcounter{footnote}{2}
\footnotetext{Empty entries indicate that the method did not report results on the dataset. The code of DT4KGR and Hi-KnowE is also not publicly available.}
\subsection{Overall Performance (RQ1)}
To address \textbf{RQ1}, we evaluate the performance of \model{} against state-of-the-art methods across five benchmark datasets. As shown in \Cref{table:overall performance}, \model{} consistently outperforms all rule-based and path-based baselines. Notably, it achieves clear gains in MRR and Hits@\{3,10\}, with relative improvements on Hits@10 of 1.9\% to 3.0\% across all datasets. These improvements highlight the benefits of combining distance-guided path collection with probabilistic modeling of both intra-path and inter-path structures. While strong path-based methods such as LoGRe and SparKGR also exploit relational paths, they lack explicit modeling of structural dependencies, limiting their effectiveness in sparse settings.
Rule-based approaches remain less competitive than the other two categories of methods, as their reliance on strict logical rules restricts generalization under sparsity. Compared with embedding-based methods, \model{} achieves competitive performance despite not relying on dense representations or model training. Its MRR scores fall within 0.1\% to 3\% of the strongest embedding-based baselines, underscoring its effectiveness as a training-free and interpretable alternative.

\subsection{Path Collection Analysis (RQ2)}
\begin{figure*}[htbp]
    \centering
    \includegraphics[width=0.7\linewidth]{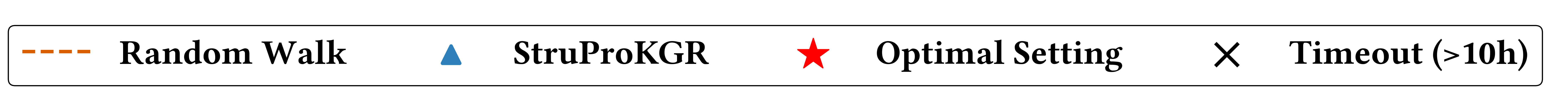}
    
    \subfigure[FB15K-237-10\%]
    {
        \includegraphics[width=.185\linewidth]{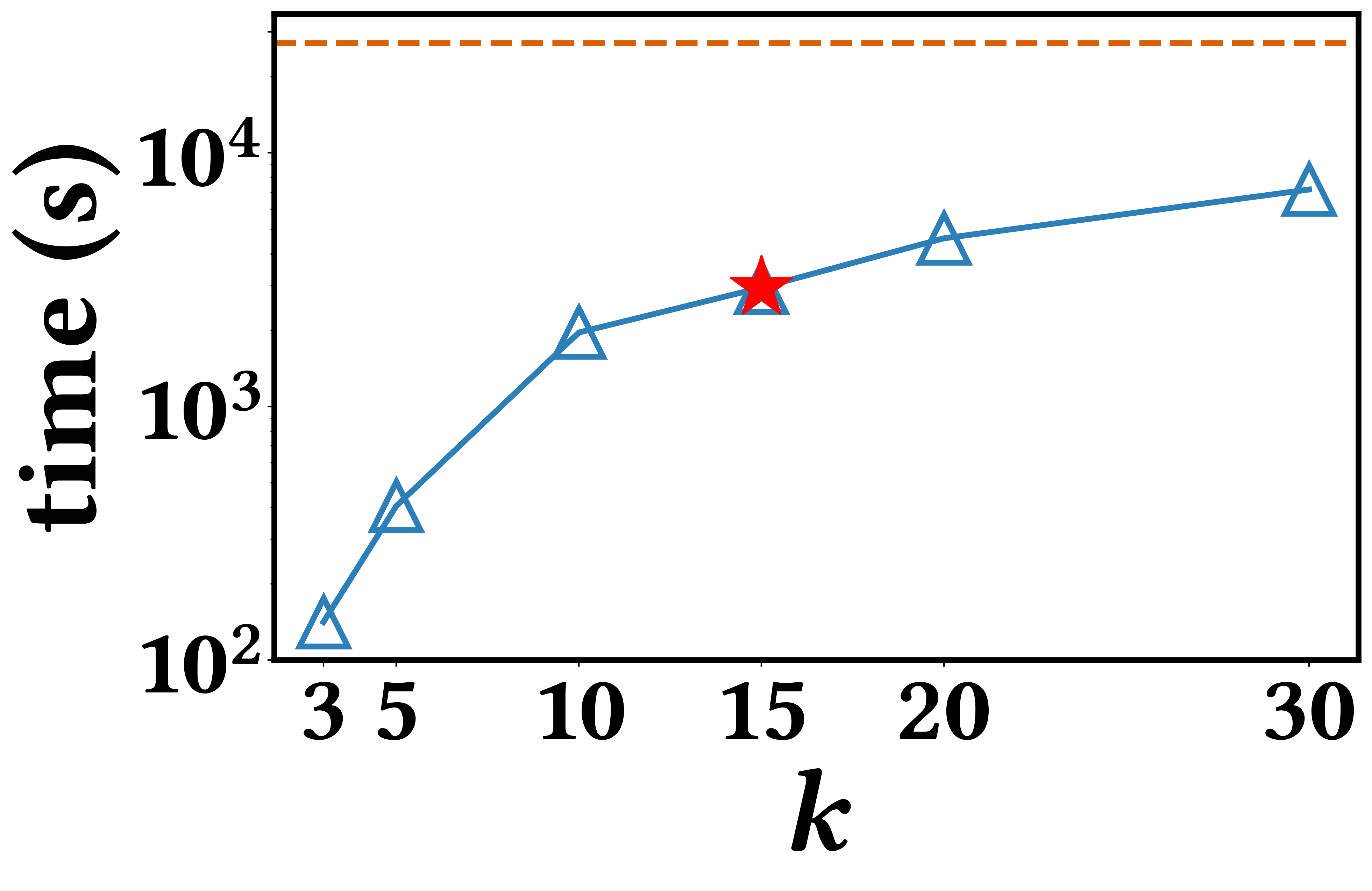}
        \label{fig:FB15K-237-10_efficiency}
    }\hspace{-0.01\linewidth}
    \subfigure[FB15K-237-20\%]
    {
        \includegraphics[width=.185\linewidth]{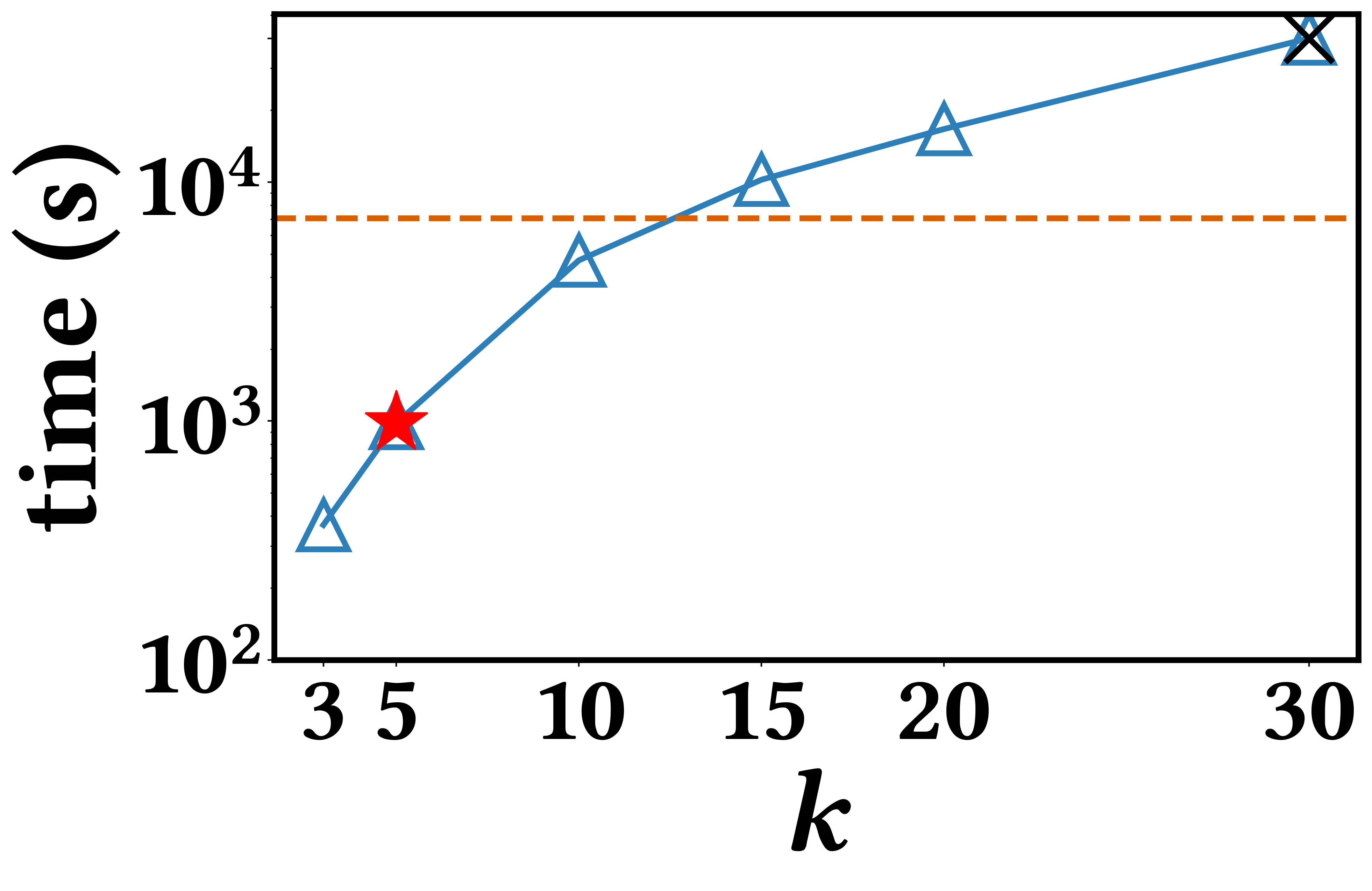}
        \label{fig:FB15K-237-20_efficiency}
    }\hspace{-0.01\linewidth}
    \subfigure[FB15K-237-50\%]
    {
        \includegraphics[width=.185\linewidth]{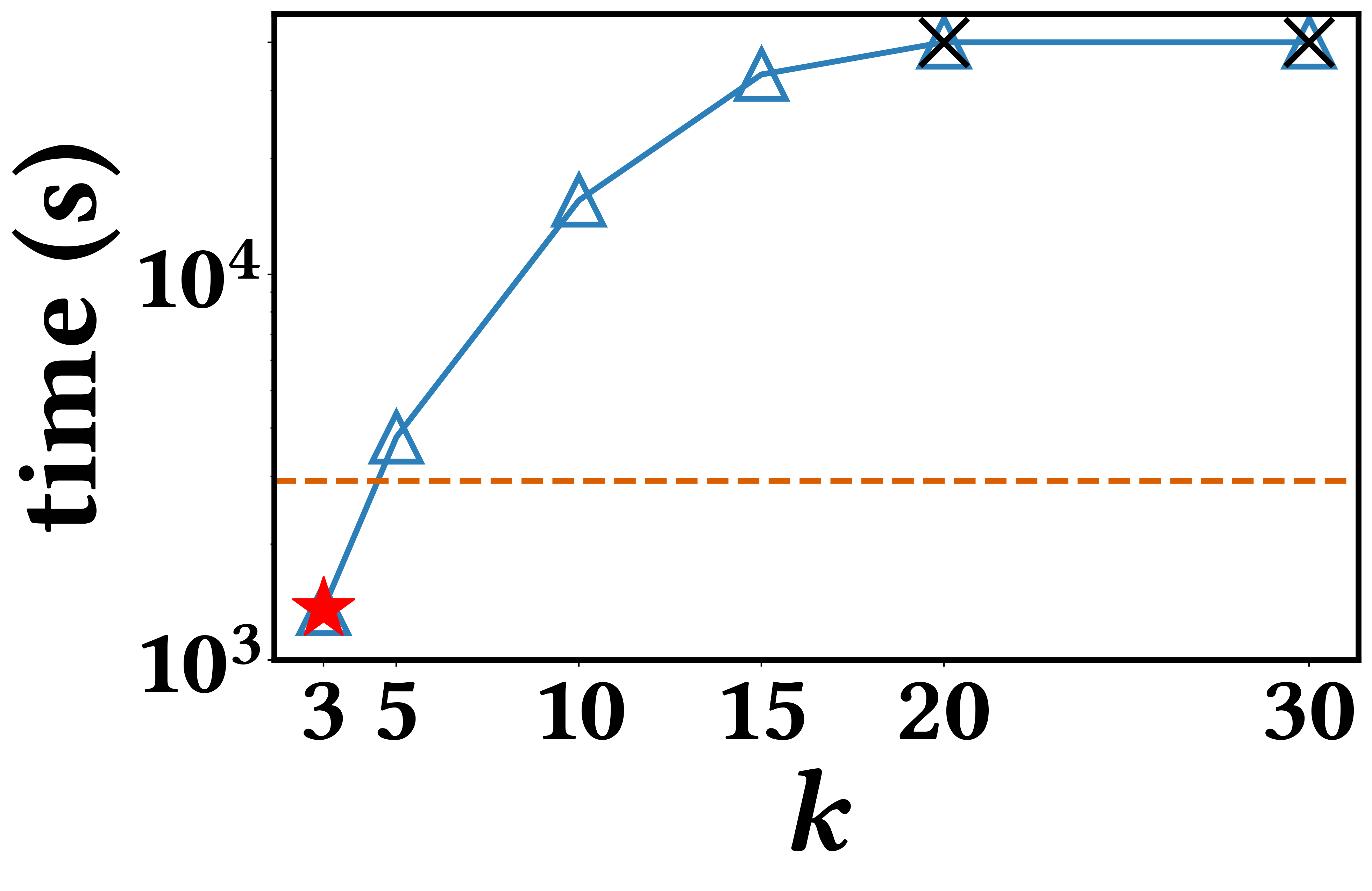}
        \label{fig:FB15K-237-50_efficiency}
    }\hspace{-0.01\linewidth}
    \subfigure[NELL23K]
    {
        \includegraphics[width=.185\linewidth]{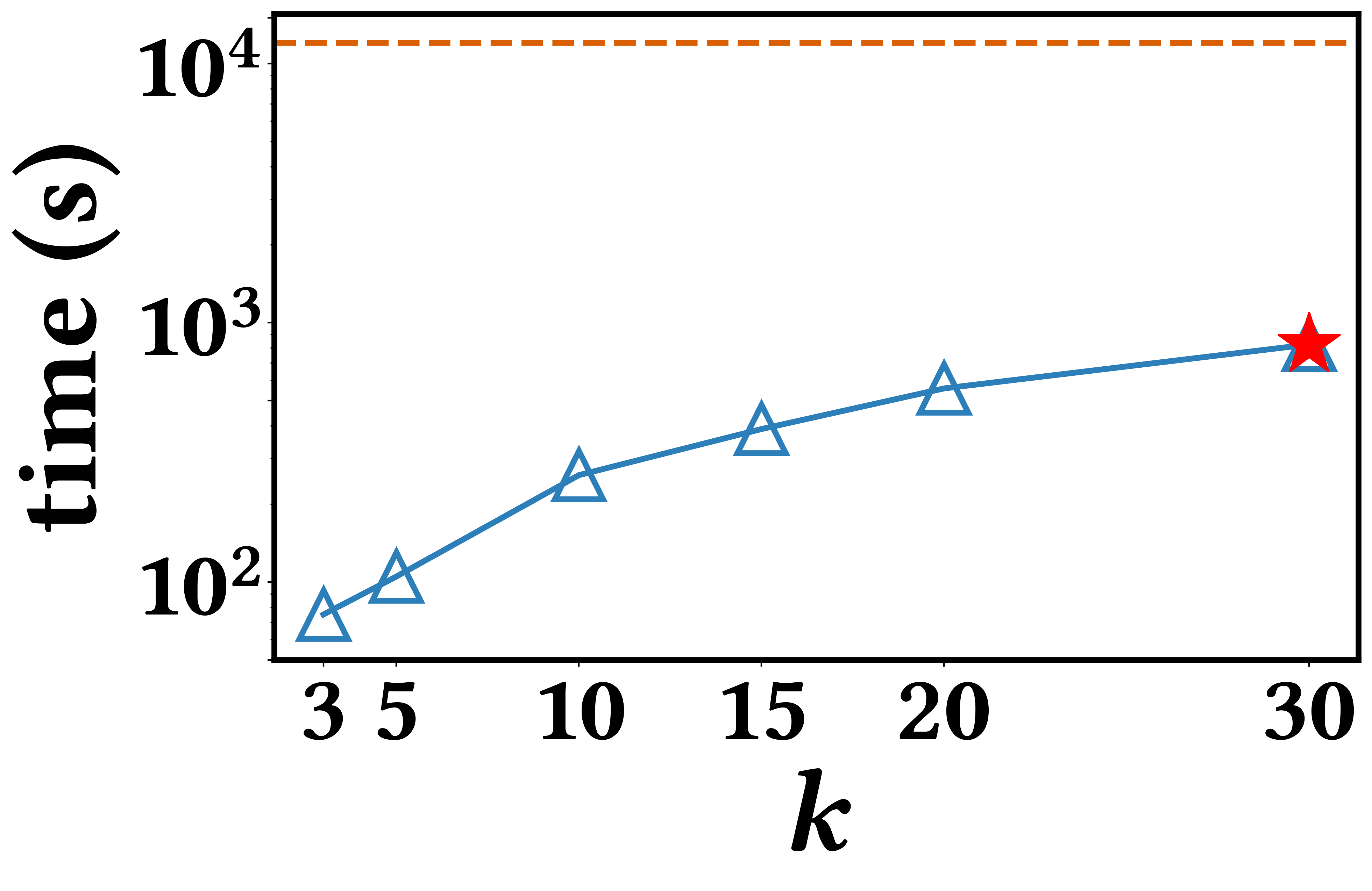}
        \label{fig:NELL23K_efficiency}
    }\hspace{-0.01\linewidth}
    \subfigure[WD-singer]
    {
        \includegraphics[width=.185\linewidth]{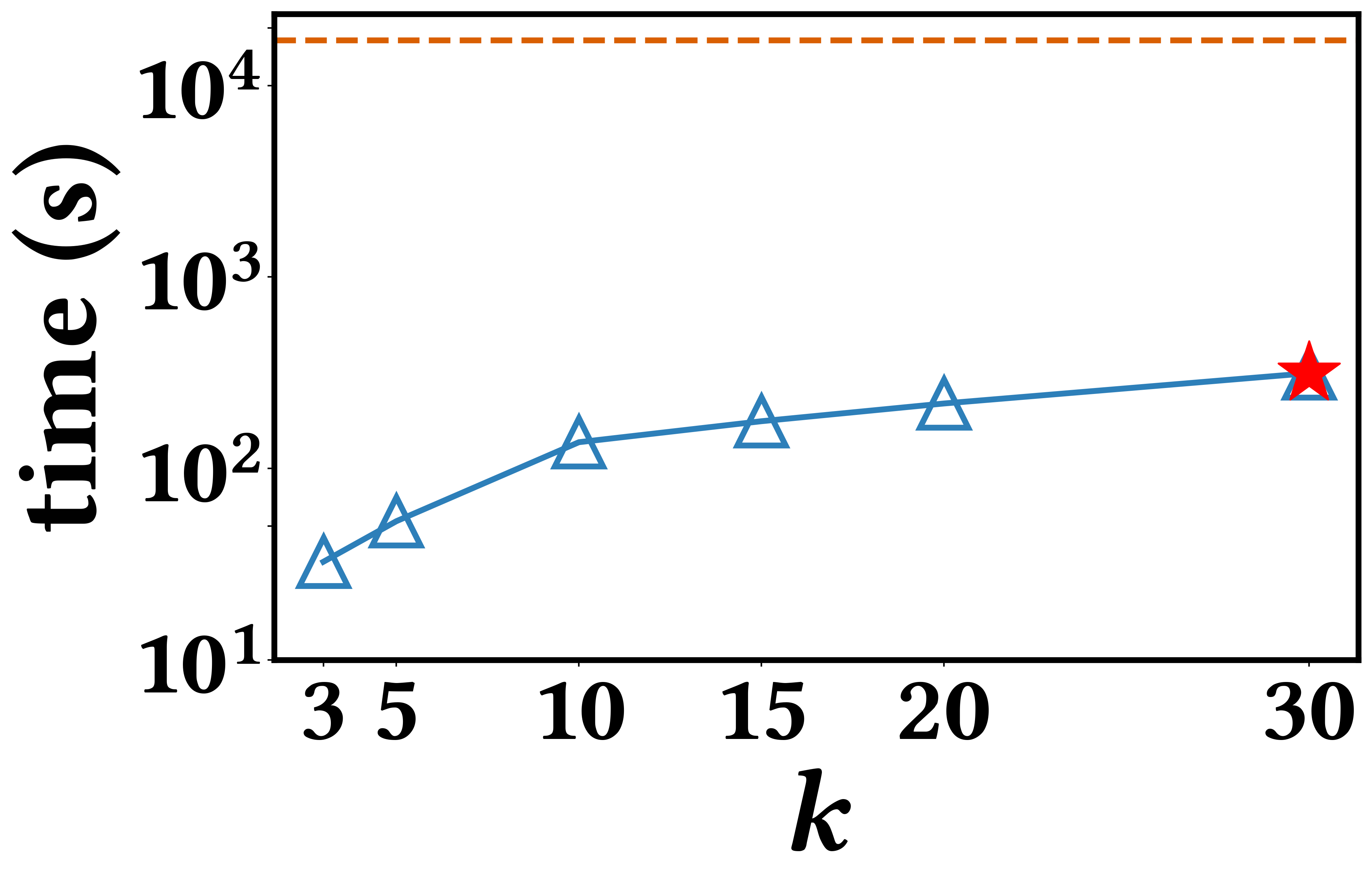}
        \label{fig:WD-singer_efficiency}
    }
    \caption{Path collection running time for \model{}$_{\text{RW}}$ and \model{} on five datasets with varying \textit{k}.}
    \label{fig:path collection efficiency}
\end{figure*}
\begin{table*}[htbp]
	\centering
    \small
	\caption{Path collection running time (s) comparison of random walk and \model{} with optimal setting.}
    \makebox[\columnwidth][c]{
        \begin{tabular}{c|ccccc}
            \toprule
            Method & FB15K-237-10\% & FB15K-237-20\% & FB15K-237-50\% & NELL23K & WD-singer \\
            \midrule
            Random Walk & 27131.76 & 7059.71 & 2917.02 & 12054.34 & 17274.81 \\
            \model{} & 2939.74 & 981.22 & 1357.07 & 824.14 & 314.47 \\
            \midrule
            Speedup & \textcolor{green!50!black}{$\uparrow\,9.22\times$} & \textcolor{green!50!black}{$\uparrow\,7.19\times$} & \textcolor{green!50!black}{$\uparrow\,2.14\times$} & \textcolor{green!50!black}{$\uparrow\,14.62\times$} & \textcolor{green!50!black}{$\uparrow\,54.93\times$} \\
            \bottomrule
        \end{tabular}
    }
	\label{tab:runtime comparison}
\end{table*}
To address \textbf{RQ2} that investigates the extent to which the distance-guided path collection enhances both the effectiveness of the overall reasoning process and the efficiency of the path collection process, a comparative analysis is conducted. 

\subsubsection{Effectiveness Analysis}
The effectiveness of the distance-guided path collection phase is assessed by comparing \model{} with a random walk-based variant, namely \model{}$_{\text{RW}}$. 
As reported in \Cref{table:path collection}, \model{} consistently outperforms \model{}$_{\text{RW}}$ across all datasets, with relative MRR gains ranging from 0.4\% on WD-singer to 3.5\% on FB15K-237-10\%.
Similar trends hold for Hits@3 and Hits@10. These improvements demonstrate that distance guidance effectively mitigates the randomness of random walks, ensuring that collected paths are more relevant for reasoning.  

\subsubsection{Efficiency Analysis}
The efficiency of the distance-guided path collection is evaluated by comparing the path collection running time of \model{} against random walk across varying values of the maximum branch number $k$, with results presented in \Cref{fig:path collection efficiency}.
As expected, the running time increases with larger $k$, and the optimal $k$ also grows with sparsity. However, it remains efficient, with timeouts (exceeding 10 hours) only occurring at $k=30$ on FB15K-237-20\% and at $k \geq 20$ on FB15K-237-50\%. 
With optimal $k$ settings, \Cref{tab:runtime comparison} shows that \model{} demonstrates efficiency improvements over the random walk approach across all datasets. Specifically, \model{} achieves up to 54.93$\times$ speedup compared to random walk, substantially reducing computational overhead while preserving accuracy.

\subsection{Ablation Study (RQ3)}
To address \textbf{RQ3} and examine the impacts of different components in the path structure-based reasoning phase, including the intra-path structure and inter-path structure, we conduct an ablation study on NELL23K and WD-singer, the two most sparse datasets.
As shown in \Cref{table:ablation}, removing both structures leads to the largest performance drop, highlighting their complementary importance. Nevertheless, even without structural modeling, \model{} remains competitive and surpasses prior methods such as DacKGR and SparKGR. Between the two, intra-path modeling contributes slightly more than inter-path, but the small performance gaps across variants indicate the overall robustness of \model{} in sparse scenarios.
\begin{table}[t]
	\setlength{\tabcolsep}{5pt}
	\centering
	\small
	\caption{Ablation study results.}
	\setlength{\tabcolsep}{6pt}
		\begin{tabular}{c|cc|cc}
			\toprule[1pt]
			\multirow{2}{*}{Method} &\multicolumn{2}{c|}{NELL23K} &\multicolumn{2}{c}{WD-singer}\\
            [-1em]\\
			\cline{2-5}\\[-0.8em]
            &MRR &Hits@3 &MRR &Hits@3 \\
			\midrule[0.5pt]
            \model{} &\textbf{0.262} &\textbf{28.5} &\textbf{0.461} &\textbf{49.8}\\
            \multicolumn{1}{r|}{\textit{w/o} structure} & 0.260 &28.2 &0.459 &49.3 \\
            \multicolumn{1}{r|}{\textit{w/o} intra-path} &0.261 &28.1 &0.459 &49.3 \\
            \multicolumn{1}{r|}{\textit{w/o} inter-path} &0.261 &28.2 &0.460 &49.7 \\
			\bottomrule[1pt]
		\end{tabular}
	\label{table:ablation}
\end{table}

\newcommand{\up}{\textcolor{green!60!black}{\(\uparrow\)}}
\newcommand{\down}{\textcolor{red!70!black}{\(\downarrow\)}}
\newcommand{\same}{\textcolor{gray}{--}}
\begin{table}[htbp]
    \centering
    \small
    \caption{Top 10 relation paths for query $(Kathy\ Cash, father, ?)$.}
    \label{tab:case study}
        \setlength{\tabcolsep}{1.5pt}
        \begin{tabular}{c|cc}
            \toprule[1pt]
    		\multicolumn{3}{c}{Query: $(Kathy\ Cash, father, ?)$}\\
    		\midrule[0.5pt]
            Path & Base Rank & Struc. Rank \\
            \midrule[0.5pt]
            $(child^{-1})$ & 1 & 1 \same \\
            $(sibling^{-1},\ father)$ & 2 & 2 \same \\
            $(sibling,\ father)$ & 3 & 3 \same \\
            $(sibling^{-1},\ child^{-1})$ & 5 & 4 \up \\
            $(mother,\ spouse)$ & 4 & 5 \down \\
            $(sibling,\ child^{-1})$ & 6 & 6 \same \\
            $(child^{-1},\ child,\ father)$ & 8 & 7 \up \\
            $(sibling,\ sibling,\ father)$ & 11 & 8 \up \\
            $(sibling^{-1},\ sibling^{-1},\ father)$ & 12 & 9 \up \\
            $(child^{-1},\ spouse)$ & 7 & 10 \down \\
            \bottomrule[1pt]
        \end{tabular}
\end{table}
\subsection{Case Study (RQ4)}
To address RQ4 and provide deeper insights into the reasoning process of the \model{}, a case study is conducted with the query $(Kathy\ Cash, father, ?)$ in WD-singer, for which the correct answer is $Johnny\ Cash$. The case study examines the top 10 paths that lead to the correct answer, focusing on how the incorporation of intra-path and inter-path structures affects path rankings. 
The paths, along with their base and structured rankings, are presented in \Cref{tab:case study}. The base rank reflects the initial ordering of $P(p|r)_{hop}$ without structural adjustments, while the structural rank accounts for the combined effects of intra-path and inter-path interactions, i.e., ranked by $P(p|r)_{inter}$.

The analysis of \Cref{tab:case study} yields several key insights. Directly relevant paths, such as $(child^{-1})$, $(sibling^{-1}, father)$, and $(sibling, father)$, consistently occupy the top ranks, showing the robustness of structural adjustments in preserving correct evidence. Less coherent paths, e.g., $(child^{-1}, spouse)$, are demoted (rank 7 $\rightarrow$ 10), while multi-hop variants like $(sibling, sibling, father)$ and $(sibling^{-1}, sibling^{-1}, father)$ are promoted (11 $\rightarrow$ 8, 12 $\rightarrow$ 9), indicating that structural modeling enhances the visibility of contextually relevant but indirect paths. Overall, \model{} effectively prioritizes the most plausible reasoning chains, improving both accuracy and interpretability.

\section{Conclusions}
\label{sec:Conclusions}
In this paper, we presented \model{}, a structural and probabilistic framework for sparse KG reasoning. Unlike previous path-based methods that rely on random walks to collect paths, it employs a distance-guided strategy to facilitate path collection, significantly reducing computational costs while improving the relevance of collected paths. 
Additionally, it incorporates probabilistic path aggregation to evaluate path reliability and utilizes the structural properties of the graph for accurate knowledge inference. 
Extensive experiments across five benchmark sparse KG reasoning datasets demonstrate its superior performance over existing path-based methods, offering a solution that is simultaneously effective, efficient, and interpretable for sparse KG reasoning.

\section*{Limitations}
Although \model{} adopts a probabilistic formulation for modeling path structures, the underlying estimates rely on empirical frequency statistics rather than true probability distributions, a constraint shared with prior works. Pure probabilistic modeling remains infeasible in sparse KGs, and frequency-based estimates may violate probability bounds. \model{} alleviates this issue through an odds-form Bayesian update, which improves numerical stability in practice.
Besides, similar to many sparse KG reasoning approaches, \model{} is primarily designed for static KGs. Updates to a dynamic KG would require recomputing path statistics and structural probabilities. However, the paths collected by \model{} carry semantic meaning. As a result, once a relation has been observed previously, the model can still make reasonable predictions for the tail entity even in a dynamically evolving KG. Developing efficient mechanisms for supporting graph updates remains an important direction for future work.


\bibliography{custom}

\newpage
\appendix
\section{Methodology Details}
\subsection{Distance-Guided Path Collection}
\label{appendix:DGPC}
\Cref{alg:distance-guided path collection} shows the process of distance-guided path collection, which integrates global distance information into local DFS expansion, enabling efficient pruning and effective prioritization of short and relevant paths.
\newcommand{\algcommentstyle}[1]{\ttfamily\textcolor{blue}{#1}}
\SetCommentSty{algcommentstyle}
\begin{algorithm*}[t]
    \caption{Distance-Guided Path Collection}
    \label{alg:distance-guided path collection}
    \KwIn{Sparse KG $\mathcal{G}_s=\{(h,r,t)|h, t\in\mathcal{E}, r\in\mathcal{R}\}$, entity type mapping function $\psi : \mathcal{E} {\rightarrow} \mathcal{C}$, maximum path length $l_{max}$, maximum branch number $k$.}
    \KwOut{Set of collected type-specific relation paths $\mathcal{P}$.}
    \tcc{Step 1. Precompute shortest distances}
    \ForEach{$u \in \mathcal{E}$}{
        Compute distance matrix $dist[u][v]$ \textbf{for all} $v$ with $dist[u][v] \leq l_{max}$\;
    }
    \tcc{Step 2. Distance-guided path collection}
    $\mathcal{P}[c][r] \gets \emptyset\ \textbf{for all } c \in \mathcal{C}, r \in \mathcal{R}$\;
    \ForEach{$(h,r,t) \in \mathcal{G}_{s}$}{
        Initialize stack $S \gets [(h, [], \{h\})]$\;
        \While{$S$ is not empty}{
            $(u, path, visited) \gets S.\text{pop()}$\;
            \If{$u = t$}{
                $\mathcal{P}[\psi(h)][r] \gets \mathcal{P}[\psi(h)][r] \cup \{\text{path}\}$\;
                \textbf{continue}\;
            }
            $nextEntities \gets []$\;
            \ForEach{$(rel, v)$ \textnormal{in adjacency list of} $u$}{
                \If{$v \notin visited$ \textnormal{\bf and} $dist[v][t] \leq l_{max} - len(path) - 1$}{
                    $nextEntities.\text{append}((rel,v))$\;
                    
                }
            }
            \If{$|nextEntities|>k$}{
                sort $nextEntities$ by $dist[v][t]$ in ascending order\;
                $nextEntities \gets nextEntities[:k]$\;
            }
            \ForEach{$(rel, v) \in \textnormal{\texttt{reverse}}(nextEntities)$}{
                $S.\text{push}((v, path + [rel], visited \cup \{v\}))$\;
            }
        }
    }
    \Return{$\mathcal{P}$}\;
\end{algorithm*}
\Cref{alg:distance-guided path collection} consists of two steps: \emph{distances precomputation} (lines 1-2) and \emph{distance-guided path collection} (lines 3-19). 
In the first step, all pairwise distances up to the maximum path length $l_{max}$ are computed.
In the second step, a bounded DFS from each triple $(h,r,t)$ is performed, maintaining at most $k$ most promising next-hop expansions at every step.
For each training triple $(h,r,t)$, we push the initial state $(h,[],\{h\})$ onto a stack $S$ (line 5) and then repeatedly pop a state $(u,path,visited)$, where $u$ represents the current entity being explored, $path$ denotes the list of relations traversed, and $visited$ tracks the set of entities already encountered to avoid cycles. (lines 6-7). 
If $u=t$, which indicates that the current path can reach the tail entity, then the path will be recorded in $\mathcal{P}[\psi(h)][r]$ as a type-specific relation path of type $\psi(h)$ and relation $r$ (lines 8-9). Subsequently, any further expansion will be skipped (line 10). 
Otherwise, we enumerate all outgoing edges $(rel,v)$ of $u$ and include in $nextEntities$, the list of candidate neighbor nodes for extending the path, only those neighbors $v \notin visited$ whose precomputed distance to $t$ satisfies
\begin{equation}
    dist[v][t] \leq l_{max}-len(path)-1.
\end{equation}
This criterion (line 13) ensures that no loops are present in the path, and each candidate $v$ has at least one path to the target within the remaining length budget.
To further narrow down the search, if the size of $nextEntities$ exceeds $k$, the maximum branch number, the candidates will be sorted by $dist[v][t]$ in ascending order and only the top $k$ will be kept (lines 15-17). 
Finally, for each $(rel,v)$ in this pruned set, we push the updated state $(v,path \cup [rel], visited\cup\{v\})$ onto $S$ in reverse order, ensuring that entities closer to $t$ are explored first (lines 18-19). 

\begin{figure}[htbp]
    \centering
    \includegraphics[width=\linewidth]{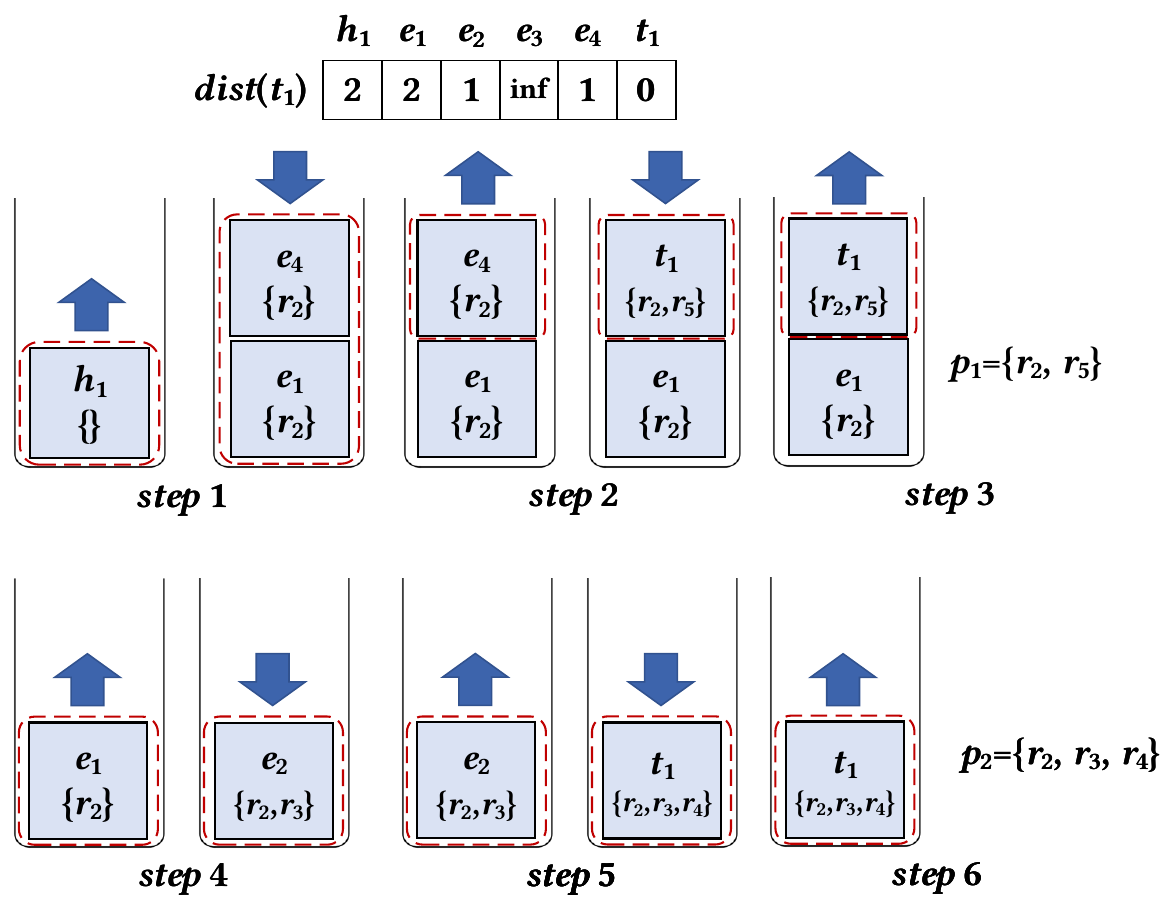}
    \caption{Illustration of distance-guided path collection.}
    \label{fig:path collection}
\end{figure}
\begin{example}
\Cref{fig:path collection} illustrates the distance-guided path collection process for the toy sparse KG depicted in \Cref{fig:framework}. Given the triple $(h_1, r_1, t_1)$ and the precomputed distance to $t_1$, as shown by the $dist(t_1)$ with values 2, 2, 1, inf, 1, and 0 for $h_1$, $e_1$, $e_2$, $e_3$, $e_4$, and $t_1$ respectively.
The process unfolds in six steps, and the $visited$ set in each state is omitted for simplicity. In the first step, $h_1$ is paired with an empty relation set $\{\}$. Its neighbors, $e_1$ and $e_4$, are then pushed onto the stack in descending order of their distance to $t_1$. Consequently, $e_4$ is explored first, as it is closer to $t_1$ than $e_1$. Subsequent steps (2-3) explore paths from $e_4$ to $t_1$, accumulating relations $\{r_2\}$ and $\{r_2, r_5\}$ respectively, forming path $p_1 = \{r_2, r_5\}$. Steps 4-6 extend the search from $e_1$ and $e_2$ to $t_1$, resulting in path $p_2 = \{r_2, r_3, r_4\}$, while $e_3$ is pruned during exploration as it is unreachable to $t_1$. 
\end{example}

\begin{lemma}
    The time complexity of \Cref{alg:distance-guided path collection} is 
    $O\bigl(|\mathcal{E}|\cdot (|\mathcal{E}|+|\mathcal{G}_{s}|) + k^{l_{max}} \cdot |\mathcal{G}_{s}|\bigr)$, where $|\mathcal{G}_{s}|$ is the number of triples in $\mathcal{G}_{s}$.
    
\end{lemma}

\begin{proof}
    The initialization of distances involves $O(|\mathcal{E}|)$ BFS operations, each taking $O(|\mathcal{E}|+|\mathcal{G}_{s}|)$ time in the worst case, i.e., each distances $\leq l_{max}$. Thus, the total precomputation cost is
    $O\bigl(|\mathcal{E}|\cdot (|\mathcal{E}|+|\mathcal{G}_{s}|)\bigr).$
    During the DFS phase, each popped state expands at most $k$ neighbors, and the maximum stack depth is $l_{max}$. In the worst case, the number of DFS states is bounded by $O(k^{l_{max}})$, giving    
    $O\bigl(k^{l_{max}} \cdot |\mathcal{G}_{s}|\bigr)$
    time in theory. 
    In practice, the BFS operates up to a maximum depth of $l_{max}$, and the distance-based pruning in line 13 significantly reduces the number of explored branches, often resulting in a much lower time complexity.
\end{proof}

\subsection{The Calculation of Path Probability and Joint Probability}
\label{appendix:Probability_Calculation}
\Cref{alg:path traversal} shows the batch search-based path traversal process when calculating path probability and joint probability. It begins by initializing a set of current entities with the head entity $h$ and an associated count of 1 (line 1). For each relation $r_i$ in the path $p = [r_1, r_2, \dots, r_n]$, it iterates through the current entities and identifies all triples in the sparse KG $\mathcal{G}_s$ that match the relation $r_i$, collecting the next set of reachable entities (lines 2-6). The counts of these entities are updated by adding the count of their predecessor entities, effectively tracking the number of ways to reach each entity (line 6). After processing all relations in the path, the algorithm returns the final set of reachable entities along with their counts (line 8), enabling efficient computation of the reachable entities that are needed for probability calculations.
\begin{algorithm}[]
    \caption{Path Traversal}
    \label{alg:path traversal}
    \KwIn{Sparse KG $\mathcal{G}_{s}$, entity $h$, path $p = [r_1, r_2, \dots, r_n]$.}
    \KwOut{Set of reachable entities $ans$.}
    $curEntities \gets \{h:1\}$\;
    \ForEach{$r_i \in p$}{
        $nextEntities \gets \{\}$\;
        \ForEach{$e_0 \in curEntities$}{
            \ForEach{$(e_0, r_i, e_1) \in \mathcal{G}_s$}{
                $nextEntities[e_1] \gets nextEntities[e_1] + curEntities[e_0]$\;
            }
        }
        $curEntities \gets nextEntities$\;
    }
    \Return $curEntities$\;
\end{algorithm}
\begin{lemma}
    The time complexity for calculating path probabilities $P(p|r)$ and joint probabilities $P(p_i, p_j|r)$ is $O(|\mathcal{G}_s| \cdot N_{paths} \cdot l_{max} + N_{pairs})$, where $N_{paths}$ is the average number of paths per relation, and $N_{pairs}$ is the number of path pairs.
\end{lemma}

\begin{proof}
    When calculating path probability, processing each triple in $\mathcal{G}_s$ involves executing up to $N_{paths}$ paths of length $l_{max}$, yielding a complexity of $O(|\mathcal{G}_s| \cdot N_{paths} \cdot l_{max})$. 
    For joint probability, although the results from the path traversals can be reused, an intersection operation is necessary for $N_{pairs}$ path pairs, contributing a complexity of $O(N_{pairs})$. Therefore, the total time cost for computing probabilities is $O(|\mathcal{G}_s| \cdot N_{paths} \cdot l_{max} + N_{pairs})$.
\end{proof}

\subsection{Path Structure-based Reasoning}
\subsubsection{Mathematical Proof for Inter-path Structure Modeling}
\label{appendix:proof_for_inter_path_structure}

\begin{algorithm*}[t]
    \caption{Path Structure-based Reasoning}
    \label{alg:path_structure_based_reasoning}
    \KwIn{Sparse KG $\mathcal{G}_{s}$, Query $(h, r, ?)$, set of relation paths $\mathcal{P}(r)$ with path probabilities $P(p|r)$, adjusted probabilities $P(p|r)_{hop}$ and joint probabilities $P(p_i, p_j|r)$, the number of explored top paths $N_{top}$.}
    \KwOut{Ranked list of candidate answers $\mathcal{A}$.}
    \tcc{Step 1. Execute top paths to gather candidates}
    $\mathcal{A} \gets \{\}$, $executedPaths \gets 0$\;
    \ForEach{$p \in \mathcal{P}(r)$}{
        $\mathcal{A}' \gets \texttt{PathTraversal}(\mathcal{G}_{s}, h,p)$\;
        \If{$\mathcal{A}' \neq \{\}$}{
            \ForEach{$(a,times) \in \mathcal{A}'$}{
                $\mathcal{A}(a).\text{append}((p,times))$\;
            }
            $executedPaths \gets executedPaths + 1$\;
            \If{$executedPaths == N_{top}$}{
                \textbf{break}\;
            }
        }
    }
    \tcc{Step 2. Calculate the probability for each candidate}
    \ForEach{$a \in \mathcal{A}$}{
        Initialize probability $P(a) \gets 0$\;
        \ForEach{$(p,times) \in a$}{
            $P(p|r)_{intra} \gets$ calculate using the intra-path structure.\;
            $P(p|r)_{inter} \gets$ calculate using the inter-path structure\;
            $P(a) \gets P(a) + P(p|r)_{inter} - P(a) \cdot P(p|r)_{inter}$\;
        }
    }
    Sort $\mathcal{A}$ by $P(a)$ in descending order\;
    \Return $\mathcal{A}$\;
\end{algorithm*}

\textbf{The Odds Form of Bayes' Theorem.} 
The standard form of Bayes' theorem is unsuitable for updating path probabilities, as these probabilities are calculated based on statistical data. This can lead to updated probabilities $P(p|r)_{inter} > 1$, which is not valid.
To overcome this issue, we utilize the odds form of Bayes' theorem, which ensures that the updated probability remains within the valid range $[0,1]$~\cite{Joyce2021bayes}. The odds form is defined as follows:
\begin{equation}
    \frac{P(A | B)}{P(\neg A | B)} = \frac{P(A)}{P(\neg A)} \times \frac{P(B | A)}{P(B | \neg A)},
\end{equation}
where $\frac{P(A)}{P(\neg A)}$ and $\frac{P(A | B)}{P(\neg A | B)}$ are known as the prior odds and the posterior odds, respectively, while $\frac{P(B | A)}{P(B | \neg A)}$ is called the likelihood ratio. 
In this context, $A$ represents the event that a specific path $p_i$ correctly infers the relation $r$, while $B$ denotes the evidence obtained from other paths.

We utilize the path probability and the joint probability to form the basis of our model. 
Strictly speaking, the interaction between paths should be captured through the likelihood ratio:
\begin{equation}
    LR(p_i, p_j) = \frac{P(p_j | p_i, r)}{P(p_j | \neg p_i, r)},
\end{equation}
where $P(p_j | p_i, r)$ is the conditional probability that $p_j$ is correct given $p_i$ is correct, and $P(p_j | \neg p_i, r)$ is the probability that $p_j$ is correct given $p_i$ is incorrect.
However, computing this ratio requires conditioning on the unobserved correctness of $p_i$, which is computationally intensive and impractical in large KGs with many paths. Pairwise computation across all paths in $\mathcal{P}(r)$ further exacerbates the scalability issue. 

\smallskip
\noindent\textbf{A Scalable Approximation.} 
To address these challenges, we propose an approximation that aggregates evidence from multiple paths while avoiding the need for exact conditional probabilities:
\begin{multline}
    LR(p_i, \mathcal{P}(r) \setminus \{p_i\}) = \\
    \frac{\sum_{p_j} P(p_i, p_{j} | r)}{\sum_{p_j} [P(p_i | r) + P(p_j | r) - P(p_i | r) \cdot P(p_j | r)]},
\end{multline}
where $p_j \in \mathcal{P}(r) \setminus \{p_i\}$ and subject to the condition $P(p_j|r)_{hop} > P(p_i|r)_{hop}$.
The ratio compares the observed joint correctness to the expected correctness under independence. A value greater than 1 suggests that the paths are more likely to be correct together than independently, indicating collaboration, while a value less than 1 suggests inhibition.
By aggregating over multiple paths rather than computing pairwise conditionals, the approach scales linearly with the number of relevant paths, making it feasible for large KGs.
The choice to include only paths $p_j$ where $P(p_j|r)_{hop} > P(p_i|r)_{hop}$ leverages the higher reliability of $p_j$ to provide more compelling evidence for updating the probability of $p_i$. By prioritizing these stronger, more trustworthy paths, the influence of noisy or less dependable signals is reduced, thereby improving the accuracy of the inference. 
Additionally, this selective focus reduces computational overhead by limiting the number of paths considered, making the process both more efficient and effective.

\begin{table*}[htbp]
\setlength{\tabcolsep}{3pt}
\centering
\small
\caption{Proportion of running time (\%) for each stage across datasets.}
\begin{tabular}{lccccc}
\toprule
Stage & FB15K-237-10\% & FB15K-237-20\% & FB15K-237-50\% & NELL23K & WD-singer \\
\midrule
Distance-Guided Path Collection & 13.2 & 6.2 & 7.3 & 29.3 & 18.3 \\
The Cal. of Path Prob. and Joint Prob. & 69.3 & 72.6 & 83.1 & 57.8 & 69.0 \\
Path Structure-based Reasoning & 17.5 & 21.2 & 9.6 & 12.9 & 12.7 \\
\bottomrule
\end{tabular}
\label{tab:time_breakdown}
\end{table*}
\subsubsection{Path Structure-based Reasoning Algorithm}
\label{appendix:reasoning_algorithm}
\Cref{alg:path_structure_based_reasoning} shows the detailed process of path structure-based reasoning. Given a query $(h,r,?)$, it first traverses the top $N_{top}$ relation paths in $\mathcal{P}(r)$ from the head entity $h$ using \texttt{PathTraversal}, collecting candidate answers $\mathcal{A}$ (lines 1-9). Next, for each candidate answer $a \in \mathcal{A}$, it initializes $P(a) = 0$ and updates $P(a)$ by aggregating path probabilities $P(p|r)_{inter}$, which is calculated by considering the intra-path structure and the inter-path structure, sequentially (lines 10-15). Finally, the candidate answers in $\mathcal{A}$ are ranked based on $P(a)$, and the sorted list is returned (lines 16-17).

\begin{lemma}
    The time complexity for path structure-based reasoning is $O(\frac{|\mathcal{A}| \cdot N_{top}^2}{2})$, where $|\mathcal{A}|$ is the number of candidate answers, and $N_{top}$ is the number of explored top paths.
\end{lemma}

\begin{proof}
    For each candidate $a \in \mathcal{A}$, the algorithm evaluates each path $p$ reaching $a$, takes at most $O(l_{max} \cdot N_{top})$ time for path traversal. For each path, it adjusts probabilities by considering paths with higher probability than it, adding $O(\frac{N_{top}^2}{2})$ complexity per candidate. Thus, the total time complexity is $O(|\mathcal{A}| \cdot (l_{max} \cdot N_{top} +  \frac{N_{top}^2}{2}))=O(\frac{|\mathcal{A}| \cdot N_{top}^2}{2})$.
\end{proof}

\section{Computational Cost Analysis}
We provide a breakdown of the computational cost across different stages of our framework. As shown in~\Cref{tab:time_breakdown}, compared with prior path-based methods, where path collection typically dominates the runtime, our distance-guided path collection (including distance precomputation) reduces the cost of this stage by restricting exploration within a bounded radius.
As a result, the primary computational overhead shifts to the calculation of path probability and joint probability, while previous works also take a similar time to conduct this process.

\end{sloppypar}
\end{document}